\documentclass[11pt]{article}
\usepackage{paper}

\usepackage{amsmath,amssymb,amsfonts,latexsym,url,xspace,algorithmic}
\usepackage{color}
\usepackage{enumerate}
\usepackage{subfigure}
\usepackage[dvips]{graphicx}

\usepackage{fullpage}

\title{Improved Spectral-Norm Bounds for Clustering\thanks{An extended abstract of this work appears in APPROX-RANDOM 2012}
}

\author{Pranjal Awasthi and Or Sheffet\thanks{This work was supported in part by the National Science Foundation under grant CCF-0830540, IIS-1065251, and CCF-1116892 as well as by CyLab at Carnegie Mellon  under grants DAAD19-02-1-0389 and W911NF-09-1-0273 from the Army  Research Office.}\\
Carnegie Mellon University\\ 
{\tt \{pawasthi,osheffet\}@cs.cmu.edu}}

\date{}

\begin{document}

\begin{titlepage}
\maketitle

\thispagestyle{empty}
\begin{abstract}
Aiming to unify known results about clustering mixtures of distributions under separation conditions, Kumar and Kannan~\cite{KumarK10} introduced a \emph{deterministic} condition for clustering datasets. They showed that this single deterministic condition encompasses many previously studied clustering assumptions. More specifically, their \emph{proximity condition} requires that in the target $k$-clustering, the projection of a point $x$ onto the line joining its cluster center $\mu$ and some other center $\mu'$, is a large additive factor closer to $\mu$ than to $\mu'$. This additive factor can be roughly described as $k$ times the spectral norm of the matrix representing the differences between the given (known) dataset and the means of the (unknown) target clustering. Clearly, the proximity condition implies \emph{center separation}~-- the distance between any two centers must be as large as the above mentioned bound.


In this paper we improve upon the work of Kumar and Kannan~\cite{KumarK10} along several axes. First, we weaken the center separation bound by a factor of $\sqrt{k}$, and secondly we weaken the  proximity condition by a factor of $k$ (in other words, the revised separation condition is independent of $k$). Using these weaker bounds we still achieve the same guarantees when all points satisfy the proximity condition. Under the same weaker bounds, we achieve \emph{even better} guarantees when only $(1-\epsilon)$-fraction of the points satisfy the condition. Specifically, we correctly cluster all but a $(\epsilon +O(1/c^4))$-fraction of the points, compared to $O(k^2\epsilon)$-fraction of~\cite{KumarK10}, which is meaningful even in the particular setting when $\epsilon$ is a constant and $k = \omega(1)$. Most importantly, we greatly simplify the analysis of Kumar and Kannan. In fact, in the bulk of our analysis we ignore the proximity condition and use only center separation, along with the simple triangle and Markov inequalities. Yet these basic tools suffice to produce a clustering which (i) is correct on all but a constant fraction of the points, (ii) has $k$-means cost comparable to the $k$-means cost of the target clustering, and (iii) has centers very close to the target centers.

Our improved separation condition allows us to match the results of the Planted Partition Model of McSherry~\cite{McSherry01}, improve upon the results of Ostrovsky et al~\cite{Ostrovsky06}, and improve separation results for mixture of Gaussian models in a particular setting.
\end{abstract}

\end{titlepage}

\section{Introduction}
In the long-studied field of clustering, there has been substantial work~\cite{Dasgupta99, DasguptaS07, Arora01, VempalaWang02, AchlioptasM05, ChaudhuriR08, KannanSV08, DHKM07, BrubakerV08} studying the problem of clustering data from mixture of distributions under the assumption that the means of the distributions are sufficiently far apart. Each of these works focuses on one particular type (or family) of distribution, and devise an algorithm that successfully clusters datasets that come from that particular type. Typically, they show that w.h.p. such datasets have certain nice properties, then use these properties in the construction of the clustering algorithm.

The recent work of Kumar and Kannan~\cite{KumarK10} takes the opposite approach. First, they define a separation condition, deterministic and thus not tied to any distribution, and show that any set of data points satisfying this condition can be successfully clustered. Having established that, they show that many previously studied clustering problems indeed satisfy (w.h.p) this separation condition. These clustering problems include Gaussian mixture-models, the Planted Partition model of McSherry~\cite{McSherry01} and the work of Ostrovsky et al~\cite{Ostrovsky06}. In this aspect they aim to unify the existing body of work on clustering under separation assumptions, proving that one algorithm applies in multiple scenarios.\footnote{We comment that, implicitly, Achlioptas and McSherry~\cite{AchlioptasM05} follow a similar approach, yet they focus only on mixtures of Gaussians and log-concave distributions. Another deterministic condition for clustering was considered by~\cite{Amin06}, which generalized the Planted Partition Model of~\cite{McSherry01}.}  

However, the attempt to unify multiple clustering works is only successful in part. First, Kumar and Kannan's analysis is ``wasteful'' w.r.t the number of clusters $k$. Clearly, motivated by an underlying assumption that $k$ is constant, their separation bound has linear dependence in $k$ and their classification guarantee has quadratic dependence on $k$. As a result, Kumar and Kannan overshoot best known bounds for the Planted Partition Model and for mixture of Gaussians by a factor of $\sqrt k$. Similarly, the application to datasets considered by Ostrovsky et al only holds for constant $k$. Secondly, the analysis in Kumar-Kannan is far from simple -- it relies on most points being ``good'', and requires multiple iterations of Lloyd steps before converging to good centers. Our work addresses these issues.

To formally define the separation condition of~\cite{KumarK10}, we require some notation. Our input consists of $n$ points in $\mathbb{R}^d$.  We view our dataset as a $n \times d$ matrix, $A$, where each datapoint corresponds to a row $A_i$ in this matrix. We assume the existence of a target partition, $T_1, T_2, \ldots, T_k$, where each cluster's center is $\mu_r = \tfrac 1 {n_r} \sum_{i\in T_r} A_i$, where $n_r = |T_r|$. Thus, the target clustering is represented by a $n \times d$ matrix of cluster centers, $C$, where $C_i = \mu_r$ iff $i\in T_r$. Therefore, the $k$-means cost of this partition is the squared Frobenius norm $\|A-C\|_F^2$, but the focus of this paper is on the spectral ($L_2$) norm of the matrix $A-C$. Indeed, the deterministic equivalent of the maximal variance in any direction is, by definition, $\frac 1 n \|A-C\|^2 = \max_{\{v: \ \|v\|=1\}} \frac 1 n \|(A-C) v\|^2$. 

\newtheorem*{definition*}{Definition}
\begin{definition*}
\label{def:proximity_condition}
Fix $i\in T_r$. We say a datapoint $A_i$ satisfies the \emph{Kumar-Kannan proximity condition} if for any $s\neq r$, when projecting $A_i$ onto the line connecting $\mu_r$ and $\mu_s$, the projection of $A_i$ is closer to $\mu_r$ than to $\mu_s$ by an additive factor of 
$\Omega \left(k(\frac 1 {\sqrt{n_r}} + \frac 1 {\sqrt{n_s}}) \|A-C\|\right)$.
\end{definition*}

Kumar and Kannan proved that if all but at most $\epsilon$-fraction of the data points satisfy the proximity condition, they can find a clustering which is correct on all but an $O(k^2\epsilon)$-fraction of the points. In particular, when $\epsilon = 0$, their algorithm clusters all points correctly. Observe, the Kumar-Kannan proximity condition gives that the distance $\|\mu_r-\mu_s\|$ is also bigger than the above mentioned bound. The opposite also holds -- one can show that if $\|\mu_r-\mu_s\|$ is greater than this bound then only few of the points do not satisfy the proximity condition.

\subsection{Our Contribution}

\paragraph{Our Separation Condition.} In this work, the bulk of our analysis is based on the following quantitatively weaker version of the proximity condition, which we call \emph{center separation}. Formally, we define 
$
\Delta_{r} =  \frac 1 {\sqrt{n_r}} \min\{ \sqrt{k}\|A-C\|, \|A-C\|_F\}
$ and we assume throughout the paper that for a large constant\footnote{We comment that throughout the paper, and much like Kumar and Kannan, we think of $c$ as a large constant ($c=100$ will do). However, our results also hold when $c = \omega(1)$, allowing for a $(1+o(1))$-approximation. We also comment that we think of $d \gg k$, so one should expect $\|A-C\|_F^2 \geq k\|A-C\|^2$ to hold, thus the reader should think of $\Delta_r$ as dependent on $\sqrt{k}\|A-C\|$. Still, including the degenerate case, where $\|A-C\|_F^2 < k\|A-C\|$, simplifies our analysis in Section~\ref{scn:Point-wise-close}. One final comment is that (much like all the work in this field) we assume $k$ is given, as part of the input, and not unknown.} $c$ we have that the means of any two clusters $T_r$ and $T_s$ satisfy \begin{equation}\|\mu_r - \mu_s\| \geq c(\Delta_{r}+\Delta_s)\label{eq:means_separation}\end{equation} 
Observe that this is a simpler version of the Kumar-Kannan proximity condition, scaled down by a factor of $\sqrt k$. Even though we show that~\eqref{eq:means_separation} gives that only a few points do not satisfy the proximity condition, our analysis (for the most part) does not partition the dataset into good and bad points, based on satisfying or non-satisfying the proximity condition. Instead, our analysis relies on basic tools, such as the Markov inequality and the triangle inequality. In that sense one can view our work as ``aligning'' Kumar and Kannan's work with the rest of clustering-under-center-separation literature -- we show that the bulk of Kannan and Kumar's analysis can be simplified to rely merely on center-separation.

\paragraph{Our results.} We improve upon the results of~\cite{KumarK10} along several axes. In addition to the weaker condition of Equation~\eqref{eq:means_separation}, we also weaken the Kumar-Kannan proximity condition by a factor of $k$, and still retrieve the target clustering, if all points satisfy the ($k$-weaker) proximity condition. Secondly, if at most $\epsilon n$ points do not satisfy the $k$-weaker proximity condition, we show that we can correctly classify all but a $(\epsilon + O(1/c^4))$-fraction of the points, improving over the bound of~\cite{KumarK10} of $O(k^2\epsilon)$. Note that our bound is meaningful even if $\epsilon$ is a constant whereas $k = \omega(1)$. Furthermore, we prove that the $k$-means cost of the clustering we output is a $(1+O(1/c))$-approximation of the $k$-means cost of the target clustering.

Once we have improved on the main theorem of Kumar and Kannan, we derive immediate improvements on its applications. In Section~\ref{subsec:ORSS-improve} we show our analysis subsumes the work of Ostrovsky et al~\cite{Ostrovsky06}, and applies also to non-constant $k$. Using the fact that Equation~\eqref{eq:means_separation} ``shaves off'' a $\sqrt{k}$ factor from the separation condition of Kumar and Kannan, we obtain a separation condition of $\Omega(\sigma_{\max} \sqrt{k})$ for learning a mixture of Gaussians, and we also match the separation results of the Planted Partition model of McSherry~\cite{McSherry01}. These results are described in Section~\ref{sec:applications}.


From an approximation-algorithms perspective, it is clear why the case of $k = \omega(1)$ is of interest, considering the ubiquity of $k$-partition problems in TCS (e.g., $k$-Median, Max $k$-coverage, Knapsack for $k$ items, maximizing social welfare in $k$-items auction -- all trivially simple for constant $k$). In addition, we comment that in our setting only the case where $k=\omega(1)$ is of interest, since otherwise one can approximate the $k$-means cost using the PTAS of Kumar et al~\cite{Kumar04}, which doesn't even require any separation assumptions. From a practical point of view, there is a variety of applications where $k$ is quite large. This includes problems such as clustering images by who is in them, clustering protein sequences by families of organisms, and problems such as deduplication where multiple databases are combined and entries corresponding to the same true entity are to be clustered together~\cite{CohenR02, Murzin_Brenner_Hubbard_Chothia_1995}. The challenges that arise from treating $k$ as a non-constant are detailed in the proofs overview (Section~\ref{subsec:discussion}).

To formally detail our results, we first define some notations and discuss a few preliminary facts.

\subsection{Notations and Preliminaries}
\label{scn:preliminaries}

The Frobenius norm of a $n\times m$ matrix $M$, denoted as $\|M\|_F$ is defined as $\|M\|_F = \sqrt{\sum_{i,j} M_{i,j}^2}$. The spectral norm of $M$ is defined as $\|M\| = \max_{x: \|x\|=1} \|Mx\|$. It is a well known fact that if the rank of $M$ is $t$, then $\|M\|_F^2 \leq t\|M\|^2$. The Singular Value Decomposition~(SVD) of $M$ is a decomposition of $M$ as $M = U \Sigma V^T$, where $U$ is a $n \times n$ unitary matrix, $V$ is a $m \times m$ unitary matrix, $\Sigma$ is a $n \times m$ diagonal matrix whose entries are nonnegative real numbers, and its diagonal entries satisfy $\sigma_1 \geq \sigma_2 \geq  \ldots \geq \sigma_{\min\{m,n\}}$. The diagonal entries in $\Sigma$ are called the singular values of $M$, and the columns of $U$ and $V$, denoted $u_i$ and $v_i$ resp., are called the left- and right-singular vectors. As a convention, when referring to singular vectors, we mean the right-singular vectors. Observe that the Singular Value Decomposition allows us to write $M = \sum_{i = 1}^{\textrm{rank}(\Sigma)} \sigma_i u_i v_i^T$. Projecting $M$ onto its top $t$ singular vectors means taking $\hat{M} = \sum_{i=1}^t \sigma_i u_i v_i^T$. It is a known fact that for any $t$, the $t$-dimensional subspace which best fits the rows of $M$, is obtained by projecting $M$ onto the subspace spanned by the top $t$ singular vectors~(corresponding to the top $t$ singular values). Another way to phrase this result is by saying that $\hat{M} = \arg\min_{N: \textrm{rank}(N) = t} \{\|M-N\|_F\}$. For a proof, see \cite{KannanV09}. The same matrix, $\hat{M}$, also minimizes the spectral norm of this difference, meaning $\hat{M} = \arg\min_{N: \textrm{rank}(N) = t} \{\|M-N\|\}$ (see~\cite{Golub96} for proof). 

As previously defined, $\|A-C\|$ denotes the spectral norm of $A-C$. The target clustering, $\mathcal{T}$, is composed of $k$ clusters $T_1, T_2, \ldots, T_k$. Observe that we use $\mu$ as an operator, where for every set $X$, we have $\mu(X) = \tfrac 1 {|X|} \sum_{i\in X} A_i$. We abbreviate, and denote $\mu_r = \mu(T_r)$. From this point on, we denote the projection of $A$ onto the subspace spanned by its top $k$-singular vectors as $\hat{A}$, and for any vector $v$, we denote $\hat v$ as the projection of $v$ onto this subspace. Throughout the paper, we abuse notation and use $i$ to iterate over the rows of $A$, whereas $r$ and $s$ are used to iterate over \emph{clusters} (or submatrices). So $A_i$ represents the $i$th row of $A$ whereas $A_r$ represents the submatrix $[A_i]_{\{i \in T_r\}}$. 

\paragraph{Basic Facts.} The analysis of our main theorem makes use of the following facts, from~\cite{McSherry01, KannanV09, KumarK10}. We advise the reader to go over the proofs, which are short, elegant, and provided in  Appendix~\ref{appx_sec:existing_thms}.
The first fact bounds the cost of assigning the points of $\hat{A}$ to their original centers.
\begin{fact} [Lemma 9 from~\cite{McSherry01}]\ \ 
\label{fct:hatA_vs_C}
$\|\hat{A} - C\|_F^2 \leq 8 \min\{k \|A-C\|^2, \|A-C\|_F^2\} \ \ \biggl( = 8 n_r \Delta_r^2\  \textrm{ for every } r\biggr)$.
\end{fact}

Next, we show that we can match each target center $\mu_r$ to a unique, relatively close, center $\nu_r$ that we get in Part I of the algorithm.
\begin{fact} [Claim 1 in Section 3.2 of~\cite{KannanV09}]
\label{fct:matching_true_and_SVD_centers}
For every $\mu_r$ there exists a center $\nu_s$ s.t.
$\|\mu_r - \nu_s \| \leq 6 \Delta_r$, so we can match each $\mu_r$ to a unique $\nu_r$.
\end{fact}


Finally, we exhibit the following fact, which is detailed in the analysis of~\cite{KumarK10}.
\begin{fact}
\label{fct:dist_centers_depends_on_pts}
Fix a target cluster $T_r$ and let $S_r$ be a set of points created by removing $\rho_{out}n_r$ points from $T_r$ and adding $\rho_{in}(s)n_r$ points from each cluster $s\neq r$, s.t. every added point $x$ satisfies $\|x-\mu_s\| \geq \tfrac 2 3 \|x-\mu_r\|$. Assume $\rho_{out} < \tfrac 1 4$ and $\rho_{in} \stackrel{\rm def} = \sum_{s\neq r} \rho_{in}(s) < \tfrac 1 4$. Then \[ \|\mu(S_r)-\mu_r\| \leq \frac 1 {\sqrt{n_r}} \left( \sqrt{\rho_{out}} + \tfrac 3 2 \sum_{s\neq r} \sqrt{\rho_{in}(s)} \right) \|A-C\| \leq \left(\sqrt{\frac {\rho_{out}} {n_r}}+ \tfrac 3 2\sqrt{k} \sqrt{\frac {\rho_{in}} {n_r}}\right) \|A-C\|\]
\end{fact}


\subsection{Formal Description of the Algorithm and Our Theorems}
\label{subsec:algorithm_description}

Having established notation, we now present our algorithm, in Figure~\ref{ALG:Kannan-Kumar}. Our algorithm's goal is three fold: (a) to find a partition that identifies with the target clustering on the majority of the points, (b) to have the $k$-means cost of this partition comparable with the target, and (c) output $k$ centers which are close to the true centers. It is partitioned into $3$ parts. Each part requires stronger assumptions, allowing us to prove stronger guarantees.


\begin{figure}[bh]
\begin{minipage}{0.9\textwidth}
\begin{center}
\begin{tabular}{|l p{5.5in}|}
\hline
\textbf{Part I}: & Find initial centers:\begin{itemize} 
\item Project $A$ onto the subspace spanned by the top $k$ singular vectors. 
\item Run a $10$-approximation algorithm\footnote{Throughout the paper, we assume the use of a $10$-approximation algorithm. Clearly, it is possible to use \emph{any} $t$-approximation algorithm, assuming $c/t$ is a large enough constant.} for the $k$-means problem on the projected matrix $\hat{A}$, and obtain $k$ centers $\nu_1, \nu_2, \ldots, \nu_k$. 
\end{itemize}\\
\textbf{Part II}: & Set $S_r \leftarrow \{i: \ \|\hat{A}_i - \nu_r\| \leq \frac 1 3 \|\hat{A}_i-\nu_s\|, \textrm{ for every } s\}$ and $\theta_r \leftarrow\mu(S_r)$.\\
& \\
\textbf{Part III}: & Repeatedly run Lloyd steps until convergence.
\begin{itemize}
\item Set $\Theta_r \leftarrow \{i: \ \|A_i - \theta_r\| \leq \|A_i-\theta_s\|, \textrm{ for every } s\}$.
\item Set $\theta_r = \mu(\Theta_r)$.\vspace{-0.2in}
\end{itemize} \\ 
\hline
\end{tabular}
\end{center}
\end{minipage}
\caption{\label{ALG:Kannan-Kumar} Algorithm \textsf{$\sim$Cluster}}
\end{figure}

\begin{itemize}
\vspace{-0.1in}
\item Assuming only the center separation of (\ref{eq:means_separation}), then \textbf{Part I} gives a clustering which (a) is correct on at least $1-O(c^{-2})$ fraction of the points \emph{from each target cluster} (Theorem~\ref{thm:point_wise_close}), and (b) has $k$-means cost smaller than $(1+O(1/c))\|A-C\|_F^2$ (Theorem~\ref{thm:small_cost}).

\item Assuming also that $\Delta_r = \frac {\sqrt{k}} {\sqrt{n_r}} \|A-C\|$, i.e. assuming the \emph{non-degenerate} case where $\|A-C\|_F^2 \geq k\|A-C\|^2$, then \textbf{Part II} finds centers that are $O(1/c) \frac {\|A-C\|}{\sqrt{n_r}}$ close to the true centers (Theorem~\ref{thm:centers_of_step_3}). As a result (see Section~\ref{subsec:proximity_condition}), if $(1-\epsilon)n$ points satisfy the proximity condition (weakened by a $k$ factor,), then we misclassify no more than $(\epsilon + O(c^{-4}))n$ points. 

\item Assuming all points satisfy the proximity condition (weakened by a $k$-factor), \textbf{Part III} finds \emph{exactly} the target partition (Theorem~\ref{thm:lloyd-steps}).
\end{itemize}



\subsection{Organization and Proofs Overview}
\label{subsec:discussion}

\paragraph{Organization.} Related work is detailed in Section~\ref{subsec:related-work}. The analysis of Part I of our algorithms is in Section~\ref{scn:Point-wise-close}. Part I is enough for us to give a ``one-line'' proof in Section~\ref{subsec:ORSS-improve} showing how the work of Ostrovsky et al falls into our framework. The analysis of Part II of the algorithm is in Section~\ref{sec:step_3}. The improved guarantees we get by applying the algorithm to the Planted Partition model and to the Gaussian mixture model are discussed in Section~\ref{sec:applications}. We conclude with an open problem in Section~\ref{scn:open_problem}.

\paragraph{Proof outline for Section~\ref{scn:Point-wise-close}.} The first part of our analysis is an immediate application of Facts~\ref{fct:hatA_vs_C} and~\ref{fct:matching_true_and_SVD_centers}. Our assumption dictates that the distance between any two centers is big ($\geq c(\Delta_r +\Delta_s)$). Part I of the algorithm assigns each projected point $\hat{A}_i$ to the nearest $\nu_r$ instead of the true center $\mu_r$ and Fact~\ref{fct:matching_true_and_SVD_centers} assures that the distance $\|\mu_r-\nu_r\|$ is small ($<6\Delta_r$). Consider a misclassified point $A_i$, where $\|A_i-\mu_r\| < \|A_i - \mu_s\|$ yet $\|\hat{A}_i - \nu_s\| < \|\hat{A}_i-\nu_r\|$. The triangle inequality assures that $\hat{A}_i$ has a fairly big distance to its true center ($ > (\tfrac c 2 - 12)\Delta_r$). We deduce that each misclassified point contributes $\Omega(c^2\Delta_r^2)$ to the $k$-means cost of assigning all projected points to their true centers. Fact~\ref{fct:hatA_vs_C} bounds this cost by $\|\hat{A}-C\|_F^2 \leq 8n_r \Delta_r^2$, so the Markov inequality proves only a few points are misclassified. Additional application of the triangle inequality for misclassified points gives that the distance between the original point $A_i$ and a true center $\mu_r$ is comparable to the distance $\|A_i - \mu_s\|$, and so assigning $A_i$ to the cluster $s$ only increases the $k$-means cost by a small factor.

\paragraph{Proof outline for Section~\ref{sec:step_3}.} In the second part of our analysis we compare between the true clustering $\mathcal{T}$ and some proposed clustering $\mathcal{S}$, looking \emph{both} at the number of misclassified points \emph{and} at the distances between the matching centers $\|\mu_r-\theta_r\|$. As Kumar and Kannan show, the two measurements are related: Fact~\ref{fct:dist_centers_depends_on_pts} shows how the distances between the means depend on the number of misclassified points, and the main lemma (Lemma~\ref{lem:small_distance_between_centers_yields_few_misclassifications}) essentially shows the opposite direction. These two relations are how Kumar and Kannan show that Lloyd steps converge to good centers, yielding clusters with few misclassified points. They repeatedly apply (their version of) the main lemma, showing that with each step the distances to the true means decrease and so fewer of the good points are misclassified.

To improve on Kumar and Kannan analysis, we improve on the two above-mentioned relations. Lemma~\ref{lem:small_distance_between_centers_yields_few_misclassifications} is a simplification of a lemma from Kumar and Kannan, where instead of projecting into a $k$-dimensional space, we project only into a $4$-dimensional space, thus reducing dependency on $k$. However, the dependency of Fact~\ref{fct:dist_centers_depends_on_pts} on $k$ is tight\footnote{In fact, Fact~\ref{fct:dist_centers_depends_on_pts} is exactly why the case of $k=\omega(1)$ is hard -- because the $L_1$ and $L_2$ norms of the vector $(\tfrac 1 {\sqrt k},\tfrac 1 {\sqrt k},\ldots,\tfrac 1 {\sqrt k})$ are not comparable for non-constant $k$.}. So in Part II of the algorithm we devise sub-clusters $S_r$ s.t. $\rho_{in}(s) = \rho_{out} / k^2$. The crux in devising $S_r$ lies in Proposition~\ref{pro:dist_hat_A_i_to_hat_mu_r} -- we show that any misclassified projected point $i\in T_s\cap S_r$ is essentially misclassified by $\hat{\mu_r}$. And since (see~\cite{AchlioptasM05}) $\|\mu_r - \hat{\mu_r}\| \leq \tfrac 1 {\sqrt k} \Delta_r$ (compared to the bound $\|\mu_r - \nu_r\| \leq 6\Delta_r$), we are able to give a good bound on $\rho_{in}(s)$. 

Recall that we rely only on center separation rather than a large batch of points satisfying the Kumar-Kannan separation, and so we do not apply iterative Lloyd steps (unless all points are good). Instead, we apply the main lemma only once, w.r.t to the misclassified points in $T_s\cap S_r$, and deduce that the distances $\|\mu_r-\theta_r\|$ are small. In other words, Part II is a single step that retrieve centers whose distances to the original centers are $\sqrt{k}$-times better than the centers retrieved by Kumar and Kannan in numerous Lloyd iterations.

\subsection{Acknowledgements}
We would like to thanks Avrim Blum for multiple helpful discussions and suggestions. We thank Amit Kumar for clarifying a certain point in the original Kumar and Kannan paper. We thank the anonymous referees for their suggestions, and especially regarding a discussion about the result of Achlioptas and McSherry.

\section{Related Work}
\label{subsec:related-work}
The work of \cite{Dasgupta99} was the first to give theoretical guarantees for the problem of learning a mixture of Gaussians under separation conditions. He showed that one can learn a mixture of $k$ spherical Gaussians provided that the separation between the cluster means is $\tilde\Omega(\sqrt{n}(\sigma_r + \sigma_s))$ and the mixing weights are not too small. Here $\sigma_r^2$ denotes the maximum variance of cluster $r$ along any direction. This separation was improved to $\tilde\Omega((\sigma_r + \sigma_s)n^{1/4})$ by \cite{DasguptaS07}. Arora and Kannan \cite{Arora01} extended these results to the case of general Gaussians. For the case of spherical Gaussians, \cite{VempalaWang02} showed that one can learn under a much weaker separation of $\tilde\Omega((\sigma_r + \sigma_s)k^{1/4})$. This was extended to arbitrary Gaussians by \cite{AchlioptasM05} and to various other distributions by \cite{KannanSV08}, although requiring a larger separation. In particular, the work of \cite{AchlioptasM05} requires a separation of $\Omega((\sigma_r + \sigma_s)(\frac 1 {\sqrt{\min(w_r,w_s)}} + \sqrt{k\log(k\min\{2^k, n\})}))$ whereas \cite{KannanSV08} require a separation of $\tilde\Omega(\frac {k^{3/2}} {w_{\min}^2} (\sigma_r + \sigma_s))$. Here $w_r$'s refer to the mixing weights. \cite{ChaudhuriR08a, ChaudhuriR08} gave algorithms for clustering mixtures of product distributions and mixtures of heavy tailed distributions. \cite{BrubakerV08} gave an algorithm for clustering the mixture of $2$ Gaussians assuming only that the two Gaussians are separated by a hyperplane. They also give results for learning a mixture of $k > 2$ Gaussians. The work of \cite{KMV10} gave an algorithm for learning a mixture of 2 Gaussians, with provably minimal assumptions. This was extended in \cite{MV10} to the case when $k > 2$ although the algorithm runs in time 
exponential in $k$. Similar results were obtained in the work of ~\cite{BS10} who can also learn more general distribution families. The work of \cite{Amin06} studied a deterministic separation condition required for efficient clustering. The precise condition presented in \cite{Amin06} is technical but essentially assumes 
that the underlying graph over the set of points has a ``low rank structure'' and presents an algorithm to recover this structure which is then enough to cluster well. In addition, previous works (e.g.~\cite{Schulman00,BBG09}) addressed the problem of clustering from the viewpoint of minimizing the number of mislabeled points.

There has been an extensive line of work on approximation algorithms for the $k$-means problem~(\cite{Ostrovsky00,Badoiu02, delaVega03, Effros04,Har-Peled04,Kanungo02}). The current best guarantee is a $(9+\epsilon)$-approximation algorithm of~\cite{Kanungo02} (with a much simpler analysis in~\cite{Kanat08}) if polynomial dependence on $k$ and the dimension $d$ is desired.\footnote{For constant $k$, \cite{Kumar04} give a PTAS for the $k$-means problem.} Another popular algorithm for $k$-means is the Lloyd's heuristics~(\cite{Lloyd82}). This heuristics, combined with a careful seeding of centers, has been shown to have good performance if the data is well separated~(see \cite{Ostrovsky06}), or to provide $O(\log(k))$-approximation in general~\cite{arthur07}. The separation-based results of~\cite{Ostrovsky06} were improved by~\cite{ABS10}.


\section{Part I of the Algorithm}
\label{scn:Point-wise-close}

In this section, we look only at Part I of our algorithm. Our approximation algorithm defines a clustering $\mathcal{Z}$, where $Z_r = \{i: \ \|\hat{A}_i - \nu_r\| \leq \|\hat{A}_i - \nu_s\| \textrm{ for every } s\}$. Our goal in this section is to show that $\mathcal{Z}$ is correct on all but a small constant fraction of the points, and furthermore, the $k$-means cost of $\mathcal{Z}$ is no more than $(1+O(1/c))$ times the $k$-means cost of the target clustering.
\begin{theorem}
\label{thm:point_wise_close}
There exists a matching (given by Fact~\ref{fct:matching_true_and_SVD_centers}) between the target clustering $\mathcal{T}$ and the clustering $\mathcal{Z} = \{Z_r\}_r$ where $Z_r = \{i: \ \|\hat{A}_i - \nu_r\| \leq \|\hat{A}_i - \nu_s\| \textrm{ for every } s\}$ that satisfies the following properties:
\begin{itemize}
\item For every cluster $T_{s_0}$ in the target clustering, no more than $O(1/c^2)|T_{s_0}|$ points are misclassified.
\item For every cluster $Z_{r_0}$ in the clustering that the algorithm outputs, we add no more than $O(1/c^2)|T_{r_0}|$ points from other clusters. 
\item At most $O(1/c^2)|T_{r_2}|$ points are misclassified overall, where $T_{r_2}$ is the second largest cluster.
\end{itemize}
\end{theorem}

\begin{proof}
Let us denote $T_{s\to r}$ as the set of points $\hat{A}_i$ that are assigned to $T_s$ in the target clustering, yet are closer to $\nu_r$ than to any other $\nu_r'$. From triangle inequality we have that
$\|\hat{A}_i-\mu_s\| \geq \|\hat{A}_i - \nu_s\| - \|\mu_s - \nu_s\|
$.
We know from Fact~\ref{fct:matching_true_and_SVD_centers} that $\|\mu_s - \nu_s\| \leq 6 \Delta_s$. Also, since $\hat{A_i}$ is closer to $\nu_r$ than to $\nu_s$, the triangle inequality gives that $2\|\hat{A}_i - \nu_s\| \geq \|\nu_r - \nu_s|$. So,
\[ \|\hat{A}_i-\mu_s\| \geq \frac 1 2 \|\nu_r - \nu_s\| - 6 \Delta_s \geq \frac 1 2\|\mu_r - \mu_s\|- 12(\Delta_r + \Delta_s)\geq \frac c 4(\Delta_r + \Delta_s) \]
Thus, we can look at $\|\hat{A}-C\|_F^2$, and using Fact~\ref{fct:hatA_vs_C} we immediately have that for every fixed $r'$
\[ \sum_r \sum_{s\neq r} |T_{s\to r}| \frac {c^2}{16} (\Delta_r + \Delta_s)^2 \leq \sum_r \sum_{i\in T_r} \|\hat{A}_i - \mu_r\|^2 = \|\hat{A}-C\|_F^2 \leq 8 n_{r'} \Delta_{r'}^2\]

The proof of the theorem follows from fixing some $r_0$ or some $s_0$ and deducing:
\begin{eqnarray*}
 \Delta_{s_0}^2 \sum_{r\neq s_0} |T_{s_0 \to r}|  \leq \sum_{r\neq s_0} |T_{s_0\to r}|(\Delta_{r} + \Delta_{s_0})^2   \leq  \sum_r \sum_{s\neq r} |T_{s\to r}| (\Delta_r + \Delta_s)^2 \leq \frac {128} {c^2} n_{s_0}\Delta_{s_0}^2\cr
\Delta_{r_0}^2 \sum_{s\neq r_0} |T_{s\to r_0}| \leq \sum_{s\neq r_0} |T_{s\to r_0}|(\Delta_{r_0} + \Delta_s)^2  \leq  \sum_r \sum_{s\neq r} |T_{s\to r}| (\Delta_r + \Delta_s)^2 \leq \frac {128} {c^2} n_{r_0}\Delta_{r_0}^2
\end{eqnarray*}
Observe that for every $r\neq s$ we have that $\Delta_r + \Delta_s \geq \Delta_{r_2}$ (where $r_2$ is the cluster with the second largest number of points), so we have that
\[\Delta_{r_2}^2 \sum_r \sum_{s\neq r} |T_{s\to r}|  \leq  \sum_r \sum_{s\neq r} |T_{s\to r}| (\Delta_r + \Delta_s)^2 \leq \frac {128} {c^2} n_{r_2}\Delta_{r_2}^2\qedhere\]
\end{proof}

We now show that the $k$-means cost of $\mathcal{Z}$ is close to the $k$-means cost of $\mathcal{T}$. Observe that the $k$-means cost of $\mathcal{Z}$ is computed w.r.t the best center of each cluster (i.e., $\mu(Z_r)$), and \emph{not} w.r.t the centers $\nu_r$.

\begin{theorem}
\label{thm:small_cost}
The $k$-means cost of $\mathcal{Z}$ is at most $(1+O(1/c))\|A-C\|_F^2$.
\end{theorem}
\begin{proof}
Given $\mathcal{Z}$, it is clear that the centers that minimize its $k$-means cost are $\mu(Z_r) = \frac 1 {|Z_r|}\sum_{i\in Z_r} A_i$. Recall that the majority of points in each $Z_r$ belong to a unique $T_r$, and so, throughout this section, we assume that all points in $Z_r$ were assigned to $\mu_r$, and not to $\mu(Z_r)$. (Clearly, this can only increase the cost.) We show that by assigning the points of $Z_r$ to $\mu_r$, our cost is at most $(1+O(1/c))\|A-C\|_F^2$, and so Theorem~\ref{thm:small_cost} follows. In fact, we show something stronger. We show that by assigning all the points in $Z_r$ to $\mu_r$, each point $A_i$ pays no more than $(1+O(1/c))\|A_i - C_i\|^2$. This is clearly true for all the points in $Z_r \cap T_r$. We show this also holds for the misclassified points.

Because $i \in T_{s\to r}$, it holds that $\|\hat A_i - \nu_r\| \leq \|\hat A_i - \nu_s\|$. Observe that for every $s$ we have that $\|A_i - \nu_s\|^2 = \|A_i - \hat A_i\|^2 + \|\hat A_i - \nu_s\|^2$, because $\hat A_i - \nu_s$ is the projection of $A_i - \nu_s$ onto the subspace spanned by the top $k$-singular vectors of $A$. Therefore, it is also true that $\|A_i - \nu_r\| \leq \| A_i - \nu_s\|$. Because of Fact~\ref{fct:matching_true_and_SVD_centers}, we have that $\|\mu_r - \nu_r\| \leq 6\Delta_r$ and $\|\mu_s - \nu_s\| \leq 6\Delta_s$, so we apply the triangle inequality and get
\begin{equation*}
\|A_i - \mu_r\| \leq \|A_i - \mu_s\| + \|\mu_r - \nu_r\| + \|\mu_s - \nu_s\| \leq \|A_i - \mu_s\| \ \left(1 + \frac { 6(\Delta_r + \Delta_s)} {\|A_i - \mu_s\|}\right) \label{eq:ratio_of_contribution}
\end{equation*}
So all we need to do is to lower bound $\|A_i - \mu_s\|$. As noted, $\|A_i - \nu_s\|\geq \|\hat{A}_i - \nu_s\|$. Thus
\[\|A_i - \mu_s\| \geq \|A_i - \nu_s\| - 6\Delta_r \geq \|\hat{A}_i - \nu_s\| - 6\Delta_r \geq \frac 1 2 \|\nu_s - \nu_r\| - 6\Delta_r \geq \frac 1 4 c(\Delta_r + \Delta_s) \] and we have the bound $\|A_i - \mu_r\| \leq \left(1 + \frac {24} c\right) \|A_i -\mu_s\|$, so $\|A_i - \mu_r\|^2 \leq \left(1 + \frac {49} c\right) \|A_i -\mu_s\|^2$.
\end{proof}

\subsection{Application: The ORSS-Separation}
\label{subsec:ORSS-improve}
One straight-forward application of Theorem~\ref{thm:small_cost} is for the datasets considered by Ostrovsky et al~\cite{Ostrovsky06}, where the optimal $k$-means cost is an $\epsilon$-fraction of the optimal $(k-1)$-means cost. Ostrovsky et al proved that for such datasets a variant of the Lloyd method converges to a good solution in polynomial time. Kumar and Kannan have shown that datasets satisfying the ORSS-separation, also have the property that most points satisfy their proximity-condition. Their analysis is not immediate, and gives a $(1+O(\sqrt{k\epsilon}))$-approximation. Here, we provide a ``one-line'' proof that Part I of Algorithm \textsf{$\sim$Cluster} yields a $(1+O(\sqrt{\epsilon}))$-approximation, for any $k$.

Suppose we have a dataset satisfying the ORSS-separation condition, so any $(k-1)$-partition of the dataset have cost $\geq \frac 1 \epsilon \|A-C\|_F^2$. For any $r$ and any $s\neq r$, by assigning all the points in $T_r$ to the center $\mu_s$, we get some $(k-1)$-partition whose cost is exactly $\|A-C\|_F^2 + n_r\|\mu_r - \mu_s\|^2$, so $\|\mu_r - \mu_s\| \geq \frac {\sqrt{\frac 1 \epsilon-1}} {\sqrt{n_r}} \|A-C\|_F$. Setting $c = O(1/\sqrt{\epsilon})$, Theorem~\ref{thm:small_cost} is immediate.

\section{Part II of the Algorithm}
\label{sec:step_3}

In this section, our goal is to show that Part II of our algorithm gives centers that are very close to the target clusters. We should note that from this point on, we assume we are in the non-degenerate case, where $\|A-C\|_F^2 \geq k\|A-C\|^2$. Therefore, $\Delta_r = \frac {\sqrt{k}}{\sqrt{n_r}} \|A-C\|$. 

Recall,  in Part II we define the sets $S_r = \{i: \|\hat{A}_i -\nu_r\| \leq \frac 1 3 \|\hat{A}_i - \nu_s\|, \ \forall s\neq r\}$. Observe, these set do not define a partition of the dataset! There are some points that are not assigned to any $S_r$. However, we only use the centers of $S_r$. We prove the following theorem.

\begin{theorem}
\label{thm:centers_of_step_3}
Denote $S_r = \{i: \|\hat{A}_i -\nu_r\| \leq \frac 1 3 \|\hat{A}_i - \nu_s\|, \ \forall s\neq r\}$. Then for every $r$ it holds that $\|\mu(S_r) - \mu_r\| = O(1/c)\ \frac 1 {\sqrt{n_r}} \|A-C\| = O(\tfrac 1 {c\sqrt{k}} \Delta_r)$.
\end{theorem}

The proof of Theorem~\ref{thm:centers_of_step_3} is an immediate application of Fact~\ref{fct:dist_centers_depends_on_pts} combined with the following two lemmas, that bound the number of misclassified points. Observe that for every point that belongs to $T_s$ yet is assigned to $S_r$ (for $s\neq r$) is also assigned to $Z_r$ in the clustering $\mathcal{Z}$ discussed in the previous section. Therefore, any misclassified point $i \in T_s\cap S_r$ satisfies that $\|A_i - \mu_r\| \leq (1+O(c^{-1}))\|A_i - \mu_s\|$ as the proof of Theorem~\ref{thm:small_cost} shows. So all conditions of Fact~\ref{fct:dist_centers_depends_on_pts} hold.

\begin{lemma}
\label{clm:S_r_cap_T_r_is_big}
Assume that for every $r$ we have that $\|\mu_r - \nu_r\|\leq 6 \Delta_r$. Then at most $\frac {512}{c^2} n_r$ points of $T_r$ do not belong to $S_r$.
\end{lemma}
\begin{lemma}
\label{clm:S_r_get_few_pts_from_T_s}
Redefine $T_{s\to r}$ as the set $T_s \cap S_r$. Assume that for every $r$ we have that $\|\mu_r - \nu_r\|\leq 6 \Delta_r$. Then for every $r$ and every $s\neq r$ we have that $|T_{s\to r}| = \left(\frac {48^2}{c^4k^2}\right)n_r$.
\end{lemma}

\begin{proof}[\textbf{Proof of Lemma~\ref{clm:S_r_cap_T_r_is_big}}]
First, we claim that if $i$ is such that $\|\hat{A}_i - \mu_r\| \leq \frac c 8 \Delta_r$, then it must be the case that $i \in S_r$.

This is a simple consequence of the triangle inequality, bounding $\|\hat{A}_i - \nu_r\| \leq \|\hat{A}_i - \mu_r\| + \|\mu_r - \nu_r\| \leq ((c/8)+6)\Delta_r$. Yet, for every $s\neq r$, the triangle inequality gives that $\|\hat{A}_i - \nu_s\| \geq \|\mu_r -\mu_s\| - \|\hat{A}_i - \mu_r\| - \|\mu_s - \nu_s\| \geq (c-\frac c 8 - 6)(\Delta_r +\Delta_s)$. Assuming $c>48$, we have that $\|\hat{A}_i-\nu_s\| \geq 3 \|\hat A_i -\nu_r\|$.

All that's left is to show that the number of $i\in T_r$ s.t. $\|\hat A_i - \mu_r\| > \frac c 8 \Delta_r$ is small. This again follows from the Markov inequality: Since $\|\hat A - C\|_F^2\leq 8k\|A-C\|^2$, then the number of such points is at most $\frac {8k\|A-C\|^2} {(c^2/64)k\|A-C\|^2} n_r$.
\end{proof}

We now turn to proving Lemma~\ref{clm:S_r_get_few_pts_from_T_s}. The general outline of the proof of Lemma~\ref{clm:S_r_get_few_pts_from_T_s} resembles to the outline of the proof of Lemma~\ref{clm:S_r_cap_T_r_is_big}. 
Proposition~\ref{pro:dist_hat_A_i_to_hat_mu_r} exhibit some property that every point in $T_{s\to r}$ must satisfy, and then we show that only few of the points in $T_s$ satisfy this property. 
Recall that $\hat{\mu_r}$ indicates the projection of $\mu_r$ onto the subspace spanned by the top $k$-singular vectors of $A$.
\vspace{-0.05in}
\begin{proposition}
\label{pro:dist_hat_A_i_to_hat_mu_r}
Fix $i\in T_s$ s.t. $\|\hat A_i - \hat{\mu}_s\| \leq 2\|\hat{A}_i-\hat{\mu}_r\|$. Then $\|\hat A_i - {\nu}_s\| < 3\|\hat{A}_i-{\nu}_r\|$, so $i \notin S_r$.
\end{proposition}
%
%
\begin{proof}
First, for every $r$ we have that $\|\hat{\mu_r}-\nu_r\| \leq \|\mu_r - \nu_r\|\leq 6\Delta_r$, as $\hat{\mu_r}-\nu_r$ is a projection of $\mu_r - \nu_r$. 

Let us fiddle with the triangle inequality, in order to obtain a lower bound on $\|\hat{A}_i - \nu_r\|$. We have that $3\|\hat A_i  -\hat \mu_r\| \geq \|\hat \mu_r - \hat \mu_s\| \geq \|\mu_r - \mu_s\| - \bigl(\|\mu_r - \nu_r\| + \|\nu_r -\hat\mu_r\|\bigr) - \bigl(\|\mu_s - \nu_s\| + \|\nu_s -\hat\mu_s\|\bigr) \geq (c-12) (\Delta_r + \Delta_s)$, thus $\|\hat{A}_i - \nu_r\| \geq \left(\frac {c-12}3 - 6\right)(\Delta_r + \Delta_s)$.

Assume for the sake of contradiction that $\|\hat A_i - {\nu}_s\| \geq 3\|\hat{A}_i-{\nu}_r\|$, and let us show this yields an upper bound on $\|\hat{A}_i - \nu_r\|$, which contradicts our lower bound. We have that
\[6\Delta_s\geq \|\hat A_i - {\nu}_s\| - \|\hat A_i - \hat{\mu}_s\| \geq 3\|\hat{A}_i-{\nu}_r\| - 2\|\hat A_i - \hat{\mu}_r\| \geq \|\hat{A}_i - \nu_r\| - 2\cdot 6\Delta_r\] It follows that $12(\Delta_r + \Delta_s) \geq \|\hat{A}_i - \nu_r\| \geq \left(\frac {c-12}3 - 6\right)(\Delta_r + \Delta_s)$. Contradiction ($c>60$).
\end{proof}
Proposition~\ref{pro:dist_hat_A_i_to_hat_mu_r}, shows that in order to bound $|T_{s\to r}|$ it suffices to bound the number of points in $T_s$ satisfying $\|\hat A_i - \hat{\mu}_s\| \geq 2\|\hat{A}_i-\hat{\mu}_r\|$. The major tool in providing this bound is the following technical lemma. This lemma is a variation on the work of~\cite{KumarK10}, on which we improve on the dependency on $k$ and simplify the proof.
\begin{lemma}[Main Lemma]
\label{lem:small_distance_between_centers_yields_few_misclassifications}
Fix $\alpha,\beta>0$. Fix $r\neq s$ and let $\zeta_r$ and $\zeta_s$ be two points s.t. $\|\mu_r - \zeta_r\|\leq \alpha \Delta_r$ and $\|\mu_s - \zeta_s\|\leq \alpha \Delta_s$. We denote $\tilde{A}_i$ as the projection of $A_i$ onto the line connecting $\zeta_r$ and $\zeta_s$. Define $X = \left\{i \in T_s\ : \  \|\tilde{A}_i-\zeta_s\|-\|\tilde{A}_i-\zeta_r\| \geq \beta \|\zeta_s - \zeta_r\| \right\}$. Then $|X| \leq 256 \frac {\alpha^2} {\beta^2} \ \frac 1 {c^4 k}\bigr(\min\left\{n_r, n_s\right\}\bigl)$.
\end{lemma}
\begin{proof}
Let $\mathcal{V}$ be the subspace spanned by the following $4$ vectors: $\{\mu_r, \mu_s, \zeta_r, \zeta_s\}$. Denote $P_{\cal V}$ as the projection onto $\mathcal{V}$. We denote $v_i=P_{\cal V}(A_i)$, and observe that $P_{\cal V}(\mu_r) = \mu_r$, and the same goes for $\mu_s$, $\zeta_r$ and $\zeta_s$. Observe also that, as a projection, $\|P_{\cal V}(A-C)\| \leq \|A-C\|$ (alternatively, $\|P_{\cal V}\| = 1$).

We now make a simple observation. Let $\bar{A}_i$ denote the projection of $A_i$ onto the line connecting $\mu_r$ and $\mu_s$. Now, the inequality $\|A_i - \mu_s\| < \|A_i - \mu_r\|$ holds iff the inequality $\|\bar{A}_i - \mu_s \| \leq \|\bar{A}_i - \mu_r\|$ holds (because $\|A_i - \mu_r\|^2 = \|A_i - \bar{A}_i\|^2 + \|\bar{A}_i - \mu_r\|^2$). Furthermore, such relation holds for any point whose projection on the line connecting $\mu_r$ and $\mu_s$ is identical to $\bar{A}_i$. In particular, if $\mathcal{W}$ is \emph{any} subspace containing $\mu_r$ and $\mu_s$, then the projection of $A_i$ onto $\mathcal{W}$ is closer to $\mu_r$ than to $\mu_s$ iff $A_i$ is closer to $\mu_r$ than to $\mu_s$.
Thus, since $\|A_i -\mu_s\| \leq \|A_i - \mu_r\|$ then $\|v_i - \mu_s\| \leq \|v_i - \mu_r\|$. Furthermore, as $\zeta_r$ and $\zeta_s$ also belong to $\mathcal{V}$, then the projection of $A_i$ onto the line connecting $\zeta_s$ and $\zeta_r$ is identical to the projection of $v_i$ onto the same line (meaning, $\tilde{A}_i = \tilde{v}_i$). So $v_i$ also satisfies the inequality: $\|\tilde{v}_i-\zeta_s\|-\|\tilde{v}_i-\zeta_r\| \geq \beta \|\zeta_s - \zeta_r\|$, and, of course, $\|v_i - \zeta_r\|^2 = \|v_i - \tilde{v}_i\|^2 + \|\tilde{v}_i-\zeta_r\|^2$.

The proof follows from upper- and lower-bounding the term $\|v_i - \zeta_s\|^2 - \|v_i - \zeta_r\|^2$. We've just shown a lower bound, as we have that 
\[\|v_i - \zeta_s\|^2 - \|v_i - \zeta_r\|^2 = \left(\|\tilde{v}_i -\zeta_s\| - \|\tilde{v}_i-\zeta_r\|\right)\ \left(\|\tilde{v}_i -\zeta_s\| + \|\tilde{v}_i-\zeta_r\|\right) \geq \beta \|\zeta_s-\zeta_r\|^2\] The triangle inequality gives that $\|v_i - \zeta_s\| \leq \|v_i - \mu_s\| + \alpha (\Delta_r+\Delta_s)$, and that $\|v_i - \zeta_r\| \geq \|v_i - \mu_r\| - \alpha (\Delta_r+\Delta_s)$, so we have the upper bound of
\begin{align*}\|v_i - \zeta_s\|^2 - \|v_i - \zeta_r\|^2 & \leq \left(\|v_i - \mu_s\| + \alpha (\Delta_r+\Delta_s)\right)^2 - \left(\|v_i - \mu_r\| - \alpha (\Delta_r+\Delta_s)\right)^2 \\ & \leq \left(\|v_i - \mu_r\| + \alpha (\Delta_r+\Delta_s)\right)^2 - \left(\|v_i - \mu_r\| - \alpha (\Delta_r+\Delta_s)\right)^2 \\ & \leq 4\alpha (\Delta_r+\Delta_s)\|v_i - \mu_r\|\end{align*}

Comparing the upper and the lower bound, we have that for any $i\in X$ the distance $\|v_i - \mu_r\| \geq \frac {\beta}{4\alpha} \frac {(c-\alpha)^2(\Delta_r + \Delta_s)^2} {\Delta_r+\Delta_s}$. As $X\subset T_s$, the Markov inequality concludes the proof \[|X| \left(\frac {c^2} 8 \frac \beta \alpha \sqrt k \|A-C\|\right)^2 \frac 1 {\min\{n_r, n_s\}} \leq \sum_{i\in T_s} \|v_i - \mu_s\|^2 \leq \|P_{\cal V}(A-C) \|_F^2 \leq 4\|A-C\|^2\qedhere\] 
\end{proof}


\begin{proof}[\textbf{Proof of Lemma~\ref{clm:S_r_get_few_pts_from_T_s}}]
Every $i\in T_{s\to r}$ must satisfy that $\|\hat A_i - \hat{\mu}_s\| \geq 2\|\hat{A}_i-\hat{\mu}_r\|$ (Proposition~\ref{pro:dist_hat_A_i_to_hat_mu_r}). Therefore, we must have that $\|\tilde{A}_i - \hat{\mu}_s\| \geq 2 \|\tilde{A}_i - \hat{\mu_r}\|$, where we denote $\tilde{A}_i$ as the projection of $A$ onto the line connecting $\hat{\mu}_r$ with $\hat{\mu}_s$ (simply because $\|\hat A_i - \hat{\mu}_s\|^2 = \|\hat A_i - \tilde{A}_i\|^2 + \|\tilde A_i - \hat{\mu}_s\|^2$.) Therefore, $\|\hat{\mu}_r - \hat{\mu}_s\| \leq \frac 3 2 \|\tilde{A}_i - \hat{\mu}_s\|$, so $\|\tilde{A}_i - \hat{\mu}_s\|-\|\tilde{A}_i - \hat{\mu}_r\| > \frac 1 3 \|\hat{\mu}_r - \hat{\mu}_s\|$. 

Thus, every $i\in T_{s\to r}$ satisfies the conditions of Lemma~\ref{lem:small_distance_between_centers_yields_few_misclassifications} with $\zeta_r = \hat{\mu}_r, \zeta_s = \hat{\mu}_s$, and $\beta = 1/3$. We deduce the $|T_{s\to r}| \leq \alpha^2~\frac {256\cdot 9} {c^4k} \min\{n_r, n_s\}$, where $\alpha$ is the bound s.t. for every $r$, $\|\mu_r - \hat{\mu}_r\| \leq \alpha \frac {\sqrt{k}}{\sqrt{n_r}}\|A-C\|$. Since $\alpha \leq \frac 1 {\sqrt{k}}$, we conclude the proof.

The fact that $\alpha$ is small was proven by Achlioptas and McSherry (Theorem $1$ of~\cite{AchlioptasM05}). Denote $u_{r}$ as the indicator vector of $T_r$. Since rank$(C) \leq k$, we get \[\|\mu_r - \hat\mu_r\| = \frac 1 {n_r} \|( A-\hat{A})^T u_{r}\| \leq \frac 1 {n_r} \|u_{r}\|\ \|A - \hat{A}\| \leq  \frac 1 {\sqrt{n_r}} \|A-C\| \qedhere\]
\end{proof}

As an interesting corollary, Theorem~\ref{thm:centers_of_step_3} dictates that for every $r$ we have that $\|\mu_r - \theta_r\| = O(1/c) \|\mu_r - \hat{\mu}_r\|$.

\subsection{The Proximity Condition -- Part III of the Algorithm}
\label{subsec:proximity_condition}

Part II of our algorithm returns centers $\theta_1, \ldots, \theta_k$ which are $O(\frac 1 {c\sqrt{n_r}}) \|A-C\|$ close to the true centers. Suppose we use these centers to cluster the points: $\Theta_s = \{i:\ \forall s', \ \|A_i - \theta_s\| \leq \|A_i - \theta_{s'}\|\}$. It is evident that this clustering correctly classifies the majority of the points. It correctly classifies any point $i \in T_s$ with $\|A_i - \mu_r\| - \|A_i - \mu_s\| = \Omega(\frac 1 {c\sqrt{n_r}}) \|A-C\|$ for every $r\neq s$, and the analysis of Theorem~\ref{thm:point_wise_close} shows that at most $O(c^{-2})$-fraction of the points do not satisfy this condition. In order to have a direct comparison with the Kumar-Kannan analysis, we now bound the number of misclassified points w.r.t the fraction of points satisfying the Kumar-Kannan proximity condition.

\begin{definition}
Denote $gap_{r,s} =  (\frac 1 {\sqrt{n_r}} + \frac 1 {\sqrt{n_s}}) \|A-C\|$. Call a point $i\in T_s$ \emph{$\gamma$-good}, if for every $r\neq s$ we have that the projection of $A_i$ onto the line connecting $\mu_r$ and $\mu_s$, denoted $\bar{A}_i$, satisfies that $\|\bar A_i - \mu_r\| - \|\bar A_i - \mu_s\| \geq \gamma\ gap_{r,s}$; otherwise we say the point is \emph{$\gamma$-bad}.
\end{definition}

\begin{corollary}
\label{cor:misclassifying_points}
If the number of $\gamma$-bad points is $\epsilon n$, then (a) the clustering $\{\Theta_1, \ldots, \Theta_k\}$ misclassifies no more than $\left(\epsilon + \frac {O(1)} {\gamma^2c^4}\right)n$ points, and (b) $\epsilon < O\left(( c-\tfrac \gamma {\sqrt k})^{-2}\right)$, assuming $\gamma< c {\sqrt k}$.
\end{corollary}
\begin{proof}
Clearly, all $\epsilon n$ bad points may be misclassified. In addition, for every $r$ and $s\neq r$, Lemma~\ref{lem:small_distance_between_centers_yields_few_misclassifications} (setting $\zeta_r = \theta_r$, $\zeta_s = \theta_s$, $\alpha = 1/c\sqrt{k}$ and $\beta = \Omega(\gamma/(c\sqrt{k}))$) proves that no more than $O(\gamma^{-2}c^{-2}k^{-1})n_s$ good points can be misclassified. Summing $\sum_{s\neq r} \tfrac 1 k n_s \leq n$, we conclude (a). 

The proof of (b) is similar to the proof of Theorem~\ref{thm:point_wise_close}. We look at the $k$-means cost of $\|\hat{A}-C\|_F^2$. We show that all $\gamma$-bad points contribute a large amount to this cost.

Take $A_i$ to be a $\gamma$-bad point from $T_s$. Projecting it down to the line connecting $\mu_r$ and $\mu_s$, we denote the projection as $\bar{A}_i$. Clearly, $\|\mu_r-\mu_s\| = \|\mu_r - \bar{A}_i\| + \|\bar A_i-\mu_s\| \geq c\sqrt{k} gap_{r,s}$ whereas $\|\mu_r - \bar{A}_i\| - \|\bar A_i-\mu_s\| \leq \gamma gap_{r,s}$. It follows that $\|\hat A_i - \mu_s\| \geq \|\bar A_i-\mu_s\| \geq \tfrac 1 2 (c\sqrt k -\gamma) gap_{r,s} \geq \frac {c\sqrt k -\gamma} {2\sqrt{n_s}}\|A-C\|$. Again, the Markov inequality gives that
\[ \#\{\textrm{bad points from }T_s\} \frac {(c\sqrt k -\gamma)^2} {4n_s}\|A-C\|^2 \leq \|\hat A - C\|_F^2 \leq 8k\|A-C\|^2 \] so from each cluster, only a fraction of $32\left(\frac {\sqrt k} {c\sqrt k - \gamma}\right)^2$ of the points can be bad.
\end{proof}

Observe that Corollary~\ref{cor:misclassifying_points} allows for multiple scaled versions of the proximity condition, based on the magnitude of $\gamma$. In particular, setting $\gamma=1$ we get a proximity condition whose bound is independent of $k$, and still our clustering misclassifies only a small fraction of the points -- at most $O(c^{-2})$ fraction of all points might be misclassified because they are $1$-bad, and no more than a $O(c^{-4})$-fraction of $1$-good points may be misclassified. In addition, if there are no $1$-bad points we show the following theorem. The proof (omitted) merely follows the Kumar-Kannan proof, plugging in the better bounds, provided by Lemma~\ref{lem:small_distance_between_centers_yields_few_misclassifications}.

\begin{theorem}
\label{thm:lloyd-steps}
Assume all data points are $1$-good. That is, for every point $A_i$ that belongs to the target cluster $T_{c(i)}$ and every $s\neq c(i)$, by projecting $A_i$ onto the line connecting $\mu_{c(i)}$ with $\mu_s$ we have that the projected point $\bar A_i$ satisfies $\|\bar A_i - \mu_{c(i)}\| - \|\bar A_i - \mu_s\| = \Omega\left((\frac 1 {\sqrt{n_{c(i)}}} + \frac 1 {\sqrt{n_s}})\right)\|A-C\|$, whereas $\|\mu_{c(i)}-\mu_s\| = \Omega\left(\sqrt{k}(\frac 1 {\sqrt{n_{c(i)}}} + \frac 1 {\sqrt{n_s}})\right)\|A-C\|$.
Then the Lloyd method, starting with $\theta_1, \ldots, \theta_k$, converges to the true centers.
\end{theorem}

\section{Applications}
\label{sec:applications}

\paragraph{Clustering a mixture of Gaussians}
 For a mixture of $k$ Gaussians, we quote the suitable results without proof, as the proof is identical to the proof in \cite{KumarK10}. We are given a mixture of $k$ Gaussians, $F_1, \ldots, F_k$, where the standard deviation of each distribution in any direction is at most $\sigma_r$, and the weight of each distribution is $w_r$. We denote $\sigma_{\max} = \max_r \{\sigma_r\}$ and $w_{\min} = \min_r \{w_r\}$.
\begin{theorem}
\label{thm:separating_gaussians}
Suppose we are given a set of $n \gg \frac {d} {w_{\min}}$ samples from a mixture of $k$ Gaussians, such that for every $r\neq s$ it holds that $\|\mu_r - \mu_s \| \geq c \sigma_{\max} \sqrt{\frac {k} {w_{\min}}} ~\emph{poly}\log \left(\frac d {w_{\min}}\right)$. Then w.h.p. these points satisfy the proximity condition.
\end{theorem}

For Gaussians, the best known separation bound is Achlioptas and McSherry's bound~\cite{AchlioptasM05} of $\Omega(\sigma_{\max}(w_{\min}^{-1/2} + \sqrt{k\log(k\cdot\min\{n,2^k\})}\ ))$. As we assume $k$ is large, this separation condition is $\tilde\Omega(\sigma_{\max}(w_{\min}^{-1/2} + \sqrt{k})) = \tilde\Omega(\sigma_{\max}/\sqrt{w_{min}})$. Therefore, the separation bound of Theorem~\ref{thm:separating_gaussians} is $\sqrt{k}$ times worse than the best known bound. However, applying Kumar and Kannan's boosting technique (Section 7 in \cite{KumarK10}), that replaces the polynomial dependency in $w_{\min}$ with a logarithmic one, we get:
\begin{theorem}
\label{thm:separating_gaussians_boosting}
Suppose we are given a set of $n \gg \frac {d} {w_{\min}}$ samples from a mixture of $k$ Gaussians, such that for every $r\neq s$ it holds that \[\|\mu_r - \mu_s \| \geq c  \sigma_{\max}\sqrt{k} ~\emph{poly}\log \left(\frac d {w_{\min}}\right)\] Then there exists an algorithm that w.h.p. correctly classifies all points.
\end{theorem}

Therefore, if for any $r$ and $r'$, both $\sigma_r \approx \sigma_{r'}$ and $w_r \approx w_{r'}$, then both~\cite{AchlioptasM05} and Theorem~\ref{thm:separating_gaussians_boosting} give roughly the same bound. If for any $r$ and $r'$ we have that $\sigma_r \approx \sigma_{r'}$, yet $w_{\min} \ll \frac 1 k$, then Theorem~\ref{thm:separating_gaussians_boosting} provides a better bound. If for any $r$ and $r'$ we have that $w_r \approx w_{r'}$, yet the directional standard deviations of the distributions vary, then the bound of~\cite{AchlioptasM05}, in which the distance between any two cluster centers depends only the parameters of these two distributions, is the better bound. If both the standard deviations and the weights vary significantly between the different distributions, then better bound is determined on a case by case basis.

\paragraph{McSherry's Planted Partition Model.} In the Planted Partition Model~\cite{McSherry01, Alon94, Alon98} our instance is a random $n$-vertex graph generated by using an implicit partition of the $n$ points into $k$ clusters. There exists an unknown $k\times k$ matrix of probabilities $P$, and for every pair of vertices $u,v$ there exists an edge connecting $u$ and $v$ w.p. $P_{rs}$ (assuming $u$ belongs to cluster $r$ and $v$ to cluster $s$). The goal here is to recover the partition of the points (thus -- recover $P$). Viewing this graph as a $n\times n$ matrix, each row is taken from a special distribution $F_r$ over $\{0,1\}^n$ -- where each coordinate $j$ is an independent Bernoulli r.v. with mean $P_{r,C(j)}$, denoting $C(j)$ as the cluster $j$ belongs to. Thus, the mean of this distribution, $\mu_r$, is a vector with its $j$-coordinate set to $P_{r,C(j)}$. Denote $w_{\min} = \min_r \{ \frac {n_r} n\}$ and $\sigma_{\max} = \max_{r,s} \sqrt{P_{rs}}$. The result of~\cite{McSherry01} is that if for every $r\neq s$ 
\begin{eqnarray}
\|\mu_r - \mu_s\| = \Omega \left( \sigma_{\max} \sqrt{k} \left( \frac 1 {w_{\min}} + \log(n/\delta) \right) \right)  \label{eq:McSherry_condition}
\end{eqnarray}
then it is possible to retrieve the partition of the vertices w.p. at least $1-\delta$.

Kumar and Kannan were not able to match the distance bounds of McSherry, and required centers to be $\sqrt{k}$ factor greater then the bound of (\ref{eq:McSherry_condition}). Here we match the bound of McSherry exactly. Following the proof in Kumar-Kannan (with few changes), we prove:
\begin{theorem}
\label{thm:McSherry_bounds}
Assuming that $\sigma_{\max} \geq \frac {3\log(n)} n$ and that the planted partition model satisfies equation~\ref{eq:McSherry_condition} for every $r\neq s$, then w.p. at least $1-\delta$, every point satisfies the proximity condition.
\end{theorem}

\begin{proof}
We follow the proof of Kumar-Kannan, making the suitable changes. McSherry (Theorem 10 of~\cite{McSherry01}) showed that w.h.p. $\|A-C\|\leq 4\sigma_{\max}\sqrt{n}$. So our goal is to show that, w.h.p., all points are $\sqrt{k}$-good. I.e., denoting $u$ as a unit-length vector connecting $\mu_r$ and $\mu_s$, we show that w.h.p. that for every $i\in T_r$ we have
\[ | (A_i - \mu_r) \cdot u | = O(\sqrt{k} \sigma_{\max} \left( \frac 1 {w_{\min}} + \log(n/\delta) \right))\]

Observe $u = \frac { \mu_s - \mu_r } { \| \mu_s - \mu_r\| }$, and due to the special structure of the means in this model, we have that $(\mu_r - \mu_s)_j = P_{rt} - P_{st}$ where $j \in T_t$. It follows that \[ \| \mu_r - \mu_s \|^2 = \sum_{t=1}^k n_t (P_{rt} - P_{st})^2\]
We therefore have
\[ | (A_i - \mu_r) \cdot u | \leq \frac 1 { \| \mu_r - \mu_s\| } \left( \sum_{t=1}^k  | P_{rt} - P_{st} | \left|\sum_{j\in T_t} A_{ij} - P_{rt}\right| \right)\]
Observe, $A_{ij}$ are i.i.d $0$-$1$ random variables with mean $P_{rt}$, so we expect their sum to deviate from its expectation by no more than a few standard deviations. Indeed, Kumar and Kannan prove that w.h.p. it holds that for every $t$ we have \[\left|\sum_{j\in T_t} A_{ij} - P_{rt}\right| \leq  B \sqrt{n_t} \sigma_{\max} \left( \frac 1 {w_{\min}} + \log(n/\delta)\right)\]
where $B$ is some sufficiently large constant. This allows us to deduce that 
\begin{align*} | (A_i - \mu_r) \cdot u | & \leq B \sigma_{\max} \left( \frac 1 {w_{\min}} + \log(n/\delta) \right) \frac { \sum_{t=1}^k \sqrt{ n_t } | P_{rt} - P_{st} | } { \sqrt{\sum_{t=1}^k n_t (P_{rt} - P_{st})^2} } \\ &  \leq B \sqrt{k} \sigma_{\max} \left( \frac 1 {w_{\min}} + \log(n/\delta) \right) \end{align*}
where the last inequality is simply the power-mean inequality.
\end{proof}

\section{An Open Problem}
\label{scn:open_problem}


Our work presents an algorithm which successfully clusters a dataset, provided that the distance between any two cluster centers meets a certain lower bound.  We would like to point out one particular direction to improve this bound. Note that our center separation bound depends on $\|A-C\|$, a property of the entire dataset. It would be nice to handle the case where the separation condition between  $\mu_r$ and $\mu_s$ depends solely on $T_r$ and $T_s$. That is, if we define $\tilde \Delta_r = \frac {\sqrt{k}} {\sqrt{n_r}} \|A_r - C_r\|$, is it possible to successfully separate clusters s.t $\|\mu_r - \mu_s\| \geq c(\tilde\Delta_r + \tilde\Delta_s)$? We comment that most of our analysis (and particularly Lemma~\ref{lem:small_distance_between_centers_yields_few_misclassifications}) builds only on the ratio between $\|\mu_r - \nu_r\|$ and $\|\mu_r - \mu_s\|$ -- we assume the first is no greater than $\alpha \Delta_r$ and that the latter is no less than $c(\Delta_r + \Delta_s)$. In fact, one can revise the proofs of Theorems \ref{thm:point_wise_close} and \ref{thm:small_cost} so that they will hold based on this assumption alone (without using the properties of the SVD). The problem therefore boils down to finding \emph{initial} centers $\{\nu_r\}$ that are sufficiently close to the true centers $\{\mu_r\}$, under the assumption that $\forall r\neq s,\ \|\mu_r -\mu_s\| \geq c(\tilde\Delta_r + \tilde\Delta_s)$. But this is an intricate task, mainly because such separation condition does \emph{not} imply that $\{\mu_1, \mu_2, \ldots, \mu_k\}$ are the centers minimizing the $k$-means cost! (Nor do $\{\hat{\mu_1}, \hat{\mu_2}, \ldots, \hat{\mu_k},\}$ minimize the $k$-means cost of $\hat{A}$.) Consider the case, for example, where cluster $r$ has very few points (say $n_r = \sqrt{n}$) and very small variance, and cluster $s$ is very big (say $n_s = n/5$), and is essentially composed of two sub-components with distance $\frac {1} {2\sqrt{n_s}} \|A_s - C_s\|$ between the centers of the two sub-components. The $k$-means cost of placing two centers within $C_s$ is smaller than placing one center at $\mu_s$ and one center at $\mu_r$. This relates to the question of designing a $t$-approximation algorithm for $k$-means, guaranteeing that \emph{each cluster's cost} cannot increase by more than a factor of $t$.

\bibliographystyle{alpha}
\bibliography{paper3}

\newcommand{\etalchar}[1]{$^{#1}$}
\begin{thebibliography}{dlVKKR03}

\bibitem[ABS10]{ABS10}
Pranjal Awasthi, Avrim Blum, and Or~Sheffet.
\newblock Stability yields a {PTAS} for k-median and k-means clustering.
\newblock In {\em FOCS}, 2010.

\bibitem[AK94]{Alon94}
Noga Alon and Nabil Kahale.
\newblock A spectral technique for coloring random 3-colorable graphs.
\newblock In {\em SIAM Journal on Computing}, pages 346--355, 1994.

\bibitem[AKS98]{Alon98}
Noga Alon, Michael Krivelevich, and Benny Sudakov.
\newblock Finding a large hidden clique in a random graph.
\newblock pages 457--466, 1998.

\bibitem[AM05]{AchlioptasM05}
Dimitris Achlioptas and Frank McSherry.
\newblock On spectral learning of mixtures of distributions.
\newblock In {\em COLT}, 2005.

\bibitem[AV07]{arthur07}
D.~Arthur and S.~Vassilvitskii.
\newblock {k-means++: The advantages of careful seeding}.
\newblock In {\em SODA}, 2007.

\bibitem[BBG09]{BBG09}
Maria-Florina Balcan, Avrim Blum, and Anupam Gupta.
\newblock Approximate clustering without the approximation.
\newblock In {\em SODA}, pages 1068--1077, 2009.

\bibitem[BHPI02]{Badoiu02}
Mihai B\={a}doiu, Sariel Har-Peled, and Piotr Indyk.
\newblock Approximate clustering via core-sets.
\newblock In {\em STOC}, pages 250--257, 2002.

\bibitem[BS10]{BS10}
Mikhail Belkin and Kaushik Sinha.
\newblock Polynomial learning of distribution families.
\newblock {\em Computing Research Repository}, abs/1004.4:103--112, 2010.

\bibitem[BV08]{BrubakerV08}
S.~Charles Brubaker and Santosh Vempala.
\newblock Isotropic pca and affine-invariant clustering.
\newblock In {\em FOCS}, 2008.

\bibitem[CO10]{Amin06}
Amin Coja-Oghlan.
\newblock Graph partitioning via adaptive spectral techniques.
\newblock {\em Comb. Probab. Comput.}, 19:227--284, 2010.

\bibitem[CR02]{CohenR02}
William~W. Cohen and Jacob Richman.
\newblock Learning to match and cluster large high-dimensional data sets for
  data integration.
\newblock In {\em KDD}, pages 475--480, 2002.

\bibitem[CR08a]{ChaudhuriR08a}
Kamalika Chaudhuri and Satish Rao.
\newblock Beyond gaussians: Spectral methods for learning mixtures of
  heavy-tailed distributions.
\newblock In {\em COLT}, 2008.

\bibitem[CR08b]{ChaudhuriR08}
Kamalika Chaudhuri and Satish Rao.
\newblock Learning mixtures of product distributions using correlations and
  independence.
\newblock In {\em COLT}, 2008.

\bibitem[Das99]{Dasgupta99}
Sanjoy Dasgupta.
\newblock Learning mixtures of gaussians.
\newblock In {\em FOCS}, 1999.

\bibitem[DHKM07]{DHKM07}
Anirban Dasgupta, John Hopcroft, Ravi Kannan, and Pradipta Mitra.
\newblock Spectral clustering with limited independence.
\newblock In {\em SODA}, 2007.

\bibitem[dlVKKR03]{delaVega03}
W.~Fernandez de~la Vega, Marek Karpinski, Claire Kenyon, and Yuval Rabani.
\newblock Approximation schemes for clustering problems.
\newblock In {\em STOC}, 2003.

\bibitem[DS07]{DasguptaS07}
Sanjoy Dasgupta and Leonard Schulman.
\newblock A probabilistic analysis of em for mixtures of separated, spherical
  gaussians.
\newblock {\em J. Mach. Learn. Res.}, 2007.

\bibitem[ES04]{Effros04}
Michelle Effros and Leonard~J. Schulman.
\newblock Deterministic clustering with data nets.
\newblock {\em ECCC}, (050), 2004.

\bibitem[GT08]{Kanat08}
Anupam Gupta and Kanat Tangwongsan.
\newblock Simpler analyses of local search algorithms for facility location.
\newblock {\em CoRR}, abs/0809.2554, 2008.

\bibitem[GVL96]{Golub96}
Gene~H. Golub and Charles~F. Van~Loan.
\newblock {\em Matrix computations (3rd ed.)}.
\newblock Johns Hopkins University Press, Baltimore, MD, USA, 1996.

\bibitem[HPM04]{Har-Peled04}
Sariel Har-Peled and Soham Mazumdar.
\newblock On coresets for $k$-means and $k$-median clustering.
\newblock In {\em STOC}, pages 291--300, 2004.

\bibitem[KK10]{KumarK10}
A.~Kumar and R.~Kannan.
\newblock Clustering with spectral norm and the k-means algorithm.
\newblock In {\em FOCS}, 2010.

\bibitem[KMN{\etalchar{+}}02]{Kanungo02}
Tapas Kanungo, David~M. Mount, Nathan~S. Netanyahu, Christine~D. Piatko, Ruth
  Silverman, and Angela~Y. Wu.
\newblock A local search approximation algorithm for $k$-means clustering.
\newblock In {\em Proc.~18th Symp.~Comp.~Geom.}, 2002.

\bibitem[KMV10]{KMV10}
Adam~Tauman Kalai, Ankur Moitra, and Gregory Valiant.
\newblock Efficiently learning mixtures of two gaussians.
\newblock In {\em STOC'10}, pages 553--562, 2010.

\bibitem[KSS04]{Kumar04}
Amit Kumar, Yogish Sabharwal, and Sandeep Sen.
\newblock A simple linear time $(1+ \epsilon)$-approximation algorithm for
  k-means clustering in any dimensions.
\newblock In {\em FOCS}, 2004.

\bibitem[KSV08]{KannanSV08}
Ravindran Kannan, Hadi Salmasian, and Santosh Vempala.
\newblock The spectral method for general mixture models.
\newblock {\em SIAM J. Comput.}, 2008.

\bibitem[KV09]{KannanV09}
Ravindran Kannan and Santosh Vempala.
\newblock Spectral algorithms.
\newblock {\em Found. Trends Theor. Comput. Sci.}, March 2009.

\bibitem[Llo82]{Lloyd82}
Stuart~P. Lloyd.
\newblock Least squares quantization in pcm.
\newblock {\em IEEE Transactions on Information Theory}, 1982.

\bibitem[MBHC95]{Murzin_Brenner_Hubbard_Chothia_1995}
A~G Murzin, S~E Brenner, T~Hubbard, and C~Chothia.
\newblock Scop: a structural classification of proteins database for the
  investigation of sequences and structures.
\newblock {\em Journal of Molecular Biology}, 247(4):536--540, 1995.

\bibitem[McS01]{McSherry01}
F.~McSherry.
\newblock Spectral partitioning of random graphs.
\newblock In {\em FOCS}, 2001.

\bibitem[MV10]{MV10}
Ankur Moitra and Gregory Valiant.
\newblock Settling the polynomial learnability of mixtures of gaussians.
\newblock In {\em FOCS'10}, 2010.

\bibitem[OR00]{Ostrovsky00}
R.~Ostrovsky and Y.~Rabani.
\newblock Polynomial time approximation schemes for geometric $k$-clustering.
\newblock In {\em FOCS}, 2000.

\bibitem[ORSS06]{Ostrovsky06}
Rafail Ostrovsky, Yuval Rabani, Leonard~J. Schulman, and Chaitanya Swamy.
\newblock The effectiveness of lloyd-type methods for the k-means problem.
\newblock In {\em FOCS}, pages 165--176, 2006.

\bibitem[Sch00]{Schulman00}
Leonard~J. Schulman.
\newblock Clustering for edge-cost minimization (extended abstract).
\newblock In {\em STOC}, pages 547--555, 2000.

\bibitem[SK01]{Arora01}
Arora Sanjeev and Ravi Kannan.
\newblock Learning mixtures of arbitrary gaussians.
\newblock In {\em STOC}, 2001.

\bibitem[VW02]{VempalaWang02}
Santosh Vempala and Grant Wang.
\newblock A spectral algorithm for learning mixtures of distributions.
\newblock In {\em Journal of Computer and System Sciences}, 2002.

\end{thebibliography}

\appendix

\section{Some Basic Lemmas}
\label{appx_sec:existing_thms}

\begin{fact} [Lemma 9 from~\cite{McSherry01}]\ \ 
\label{appx_fct:hatA_vs_C}
$\|\hat{A} - C\|_F^2 \leq 8 \min\{k \|A-C\|^2, \|A-C\|_F^2\} \ \ \biggl( = 8 n_r \Delta_r^2\  \textrm{ for every } r\biggr)$.
\end{fact}
\begin{proof}
\[\|\hat{A} - C\|_F^2  \leq  2k\|\hat{A}-C\|^2 \leq  2k \left( \|\hat{A} - A\| + \| A-C\|\right)^2 \leq  2k \left( 2 \|A-C\|\right)^2\]where the first inequality holds because rank$(\hat{A} - C)\leq 2k$, and the last inequality follows from the fact that $\hat{A} = \arg\min_{N: \textrm{rank}(N) = k} \{\|A-N\|\}$. For the same reason, $\|\hat{A} - C\|_F \leq \|A - \hat{A} \|_F + \|A-C\|_F \leq 2\|A-C\|_F$.
\end{proof}

\begin{fact} [Claim 1 in Section 3.2 of~\cite{KannanV09}]
\label{appx_fct:matching_true_and_SVD_centers}
For every $\mu_r$ there exists a center $\nu_s$ s.t.
$\|\mu_r - \nu_s \| \leq 6 \Delta_r$, so we can match each $\mu_r$ to a unique $\nu_r$.
\end{fact}
\begin{proof}
Observe that by taking $\hat{A}-\hat{C}$, we project $A-C$ to a $k$-dimensional subspace, so we have that $\|\hat{A}-\hat{C}\|_F^2 \leq k\|\hat{A}-\hat{C}\|^2 \leq k\|A-C\|^2$. Similarly, $\|\hat{A}-\hat{C}\|_F^2 \leq \|A-C\|_F^2$. 

Assume for the sake of contradiction that $\exists r$ s.t. $\|\mu_r - \nu_s \| > 6 \Delta_r$ for all $s$. Since $\|\hat{A}-\hat{C}\|_F^2 \leq n_r \Delta_r^2$, then our $10$-approximation algorithm yields a clustering of cost $\leq 10 n_r \Delta_r^2$. In contrast, as each $\hat{A}_i$ is assigned to some $\nu_{c(i)}$, the contribution of only the points in $T_r$ to the $k$-means cost of the clustering is more than \[\sum_{i \in T_r} \left\|(\mu_r-\nu_{c(i)}) - (\hat{A}_i - \mu_r)\right\|^2  > \frac {n_r} 2 \left( 6\Delta_r\right)^2 - \sum_{i \in T_r}\|\hat{A}_i-\mu_r\|^2 \geq 18 n_r \Delta_r^2 - \|\hat{A} - C\|_F^2 \geq 10 n_r \Delta_r^2\] where the first inequality follows from the fact that $(a-b)^2 \geq \frac 1 2 a^2 - b^2$. 
\end{proof}

Now, in order to prove Fact~\ref{fct:dist_centers_depends_on_pts} (also cited below as Fact~\ref{appx_fct:dist_centers_depends_on_pts}), we need the following Fact.
\begin{fact}[Lemma 5.2 and Corollary 5.3 from \cite{KumarK10}]
\label{appx_fct:dist_of_means_of_subsets}
Fix any cluster $T_r$ and a subset $X\subset T_r$. Then \[|X|\ \|\mu(X) - \mu_r\| = (|T_r| - |X|) \ \|\mu(T_r \setminus X) - \mu_r\| \leq \sqrt{|X|} \ \|A_r-C_r\| \]
\end{fact}
\begin{proof}
Let $u_X$ be the indicator vector of $X$. Then \[\|\ |X|\ (\mu(X)-\mu_r)\ \| = \|(A_r-C_r)^T\ u_X\| \leq \|(A_r-C_r)^T\|\ \|u_X\| = \|A_r-C_r\| \sqrt{|X|} \] and the fact that $|X|\ \|\mu(X) - \mu_r\| = |T_r \setminus X|\ \|\mu(T_r \setminus X) - \mu_r\|$ is simply because $\mu_r = \frac {|X|}{|T_r|} \mu(X) + \frac {|T_r\setminus X|}{|T_r|} \mu(T_r\setminus X)$. 
\end{proof}

\begin{fact}
\label{appx_fct:dist_centers_depends_on_pts}
Fix a target cluster $T_r$ and let $S_r$ be a set of points created by removing $\rho_{out}n_r$ points from $T_r$ and adding $\rho_{in}(s)n_r$ points from each cluster $s\neq r$, s.t. every added point $x$ satisfies $\|x-\mu_s\| \geq \tfrac 2 3 \|x-\mu_r\|$. Assume $\rho_{out} < \tfrac 1 4$ and $\rho_{in} \stackrel{\rm def} = \sum_{s\neq r} \rho_{in}(s) < \tfrac 1 4$. Then \[ \|\mu(S_r)-\mu_r\| \leq \frac 1 {\sqrt{n_r}} \left( \sqrt{\rho_{out}} + \tfrac 3 2 \sum_{s\neq r} \sqrt{\rho_{in}(s)} \right) \|A-C\| \leq \left(\sqrt{\frac {\rho_{out}} {n_r}}+ \tfrac 3 2\sqrt{k} \sqrt{\frac {\rho_{in}} {n_r}}\right) \|A-C\|\]
\end{fact}
\begin{proof}
We break $\|\mu(S_r)-\mu_r\|$ into its components and deduce
\begin{eqnarray*}
& \|\mu(S_r)-\mu_r\| & \leq \frac {(1-\rho_{out})n_r}{n_r} \|\mu(S_r\cap T_r) - \mu_r\|  + \sum_{s\neq r} \frac {\rho_{in}(s)n_r}{n_r} \|\mu(S_r\cap T_s) - \mu_r\| \cr 
& & \leq \frac {(1-\rho_{out})n_r}{n_r} \|\mu(S_r\cap T_r) - \mu_r\|  + \tfrac 3 2 \sum_{s\neq r} \frac {\rho_{in}(s)n_r}{n_r} \|\mu(S_r\cap T_s) - \mu_s\| \cr
\end{eqnarray*}
Plugging in Fact~\ref{appx_fct:dist_of_means_of_subsets} we have $\|\mu(S_r)-\mu_r\| \leq \frac 1 {n_r} \left( \sqrt{\rho_{out}n_r} + \tfrac 3 2\sum_{s\neq r} \sqrt{\rho_{in}(s)n_r} \right) \|A-C\|$. The last inequality comes from maximizing the sum of square-roots by taking each $\rho_{in}(s) = \rho_{in}/k$.
\end{proof}

\end{document}